\documentclass[english]{article}
\usepackage{mathpazo}
\usepackage[T1]{fontenc}
\usepackage[latin9]{inputenc}
\usepackage{color}
\usepackage{float}
\usepackage{calc}
\usepackage{amsthm}
\usepackage{amsmath}
\usepackage{graphicx}

\makeatletter

\newenvironment{lyxgreyedout}{\textcolor[gray]{0.8}\bgroup}{\egroup}

\theoremstyle{plain}
\theoremstyle{plain}
\newtheorem{thm}{Theorem}

\usepackage{nips07submit_e,times}

\makeatother

\usepackage{babel}

\begin{document}

\title{Maximin affinity learning of image segmentation}

\author{Srinivas C. Turaga \thanks{\texttt{sturaga@mit.edu}}\\
MIT\\
\And Kevin L. Briggman \\
Max-Planck Insitute for Medical Research \\
\And Moritz Helmstaedter \\
Max-Planck Insitute for Medical Research \\
\And Winfried Denk \\
Max-Planck Insitute for Medical Research \\
\And H. Sebastian Seung \\
MIT, HHMI}
\maketitle
\begin{abstract}
Images can be segmented by first using a classifier to predict an
affinity graph that reflects the degree to which image pixels must
be grouped together and then partitioning the graph to yield a segmentation.
Machine learning has been applied to the affinity classifier to produce
affinity graphs that are good in the sense of minimizing edge misclassification
rates. However, this error measure is only indirectly related to the
quality of segmentations produced by ultimately partitioning the affinity
graph. We present the first machine learning algorithm for training
a classifier to produce affinity graphs that are good in the sense
of producing segmentations that directly minimize the Rand index,
a well known segmentation performance measure.

The Rand index measures segmentation performance by quantifying the
classification of the connectivity of image pixel pairs after segmentation.
By using the simple graph partitioning algorithm of finding the connected
components of the thresholded affinity graph, we are able to train
an affinity classifier to directly minimize the Rand index of segmentations
resulting from the graph partitioning. Our learning algorithm corresponds
to the learning of maximin affinities between image pixel pairs, which
are predictive of the pixel-pair connectivity.
\end{abstract}

\section{Introduction}

Supervised learning has emerged as a serious contender in the field
of image segmentation, ever since the creation of training sets of
images with {}``ground truth'' segmentations provided by humans,
such as the Berkeley Segmentation Dataset \cite{Martin:2001}. Supervised
learning requires 1) a parametrized algorithm that map images to segmentations,
2) an objective function that quantifies the performance of a segmentation
algorithm relative to ground truth, and 3) a means of searching the
parameter space of the segmentation algorithm for an optimum of the
objective function. 

In the supervised learning method presented here, the segmentation
algorithm consists of a parametrized \emph{classifier} that predicts
the weights of a nearest neighbor affinity graph over image pixels,
followed by a graph \emph{partitioner} that thresholds the affinity
graph and finds its connected components. Our objective function is
the Rand index \cite{Rand:1971}, which has recently been proposed
as a quantitative measure of segmentation performance \cite{Unnikrishnan:2007}.
We {}``soften'' the thresholding of the classifier output and adjust
the parameters of the classifier by gradient learning based on the
Rand index. 

Because maximin edges of the affinity graph play a key role in our
learning method, we call it \emph{maximin affinity learning of image
segmentation}, or MALIS. The minimax path and edge are standard concepts
in graph theory, and maximin is the opposite-sign sibling of minimax.
Hence our work can be viewed as a machine learning application of
these graph theoretic concepts. MALIS focuses on improving classifier
output at maximin edges, because classifying these edges incorrectly
leads to genuine segmentation errors, the splitting or merging of
segments.

To the best of our knowledge, MALIS is the first supervised learning
method that is based on optimizing a genuine measure of segmentation
performance. The idea of training a classifier to predict the weights
of an affinity graph is not novel. Affinity classifiers were previously
trained to minimize the number of misclassified affinity edges \cite{Fowlkes:2003,Martin:2004}.
This is not the same as optimizing segmentations produced by partitioning
the affinity graph. There have been attempts to train affinity classifiers
to produce good segmentations when partitioned by normalized cuts
\cite{Meila:2001,Bach:2006}. But these approaches do not optimize
a genuine measure of segmentation performance such as the Rand index.
The work of Bach and Jordan \cite{Bach:2006} is the closest to our
work. However, they only minimize an upper bound to a renormalized
version of the Rand index. Both approaches require many approximations
to make the learning tractable.

In other related work, classifiers have been trained to optimize performance
at detecting image pixels that belong to object boundaries \cite{Martin:2004,Dollar:2006,Maire:2008}.
Our classifier can also be viewed as a boundary detector, since a
nearest neighbor affinity graph is essentially the same as a boundary
map, up to a sign inversion. However, we combine our classifier with
a graph partitioner to produce segmentations. The classifier parameters
are not trained to optimize performance at boundary detection, but
to optimize performance at segmentation as measured by the Rand index. 

There are also methods for supervised learning of image labeling using
Markov or conditional random fields \cite{He:2004}. But image labeling
is more similar to multi-class pixel classification rather than image
segmentation, as the latter task may require distinguishing between
multiple objects in a single image that all have the same label.

In the cases where probabilistic random field models have been used
for image parsing and segmentation, the models have either been simplistic
for tractability reasons \cite{Kumar:2003} or have been trained piecemeal.
For instance, Tu et al. \cite{Tu:2005} separately train low-level
discriminative modules based on a boosting classifier, and train high-level
modules of their algorithm to model the joint distribution of the
image and the labeling. These models have never been trained to minimize
the Rand index.

\section{Partitioning a thresholded affinity graph by connected components}

\begin{figure}
\begin{centering}
\includegraphics[width=1\columnwidth]{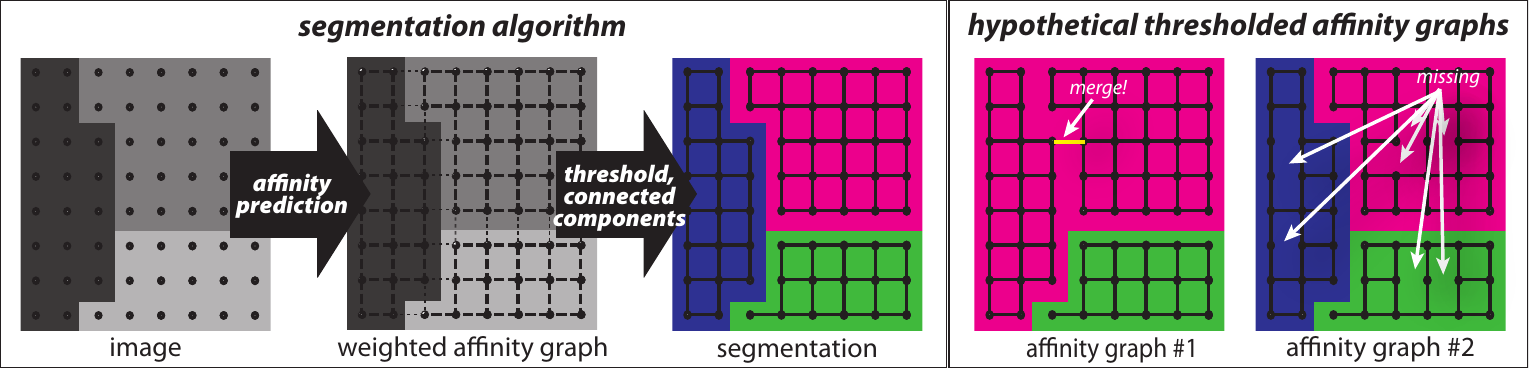}
\par\end{centering}

\caption{(left) \textbf{Our segmentation algorithm.} We first generate a nearest
neighbor weighted affinity graph representing the degree to which
nearest neighbor pixels should be grouped together. The segmentation
is generated by finding the connected components of the thresholded
affinity graph. (right) \textbf{Affinity misclassification rates are
a poor measure of segmentation performance. }Affinity graph \#1 makes
only 1 error (dashed edge) but results in poor segmentations, while
graph \#2 generates a perfect segmentation despite making many affinity
misclassifications (dashed edges).\label{fig:Segmentation-algorithm-two-graphs}}

\end{figure}
Our class of segmentation algorithms is constructed by combining a
classifier and a graph partitioner (see Figure \ref{fig:Segmentation-algorithm-two-graphs}).
The classifier is used to generate the weights of an affinity graph.
The nodes of the graph are image pixels, and the edges are between
nearest neighbor pairs of pixels. The weights of the edges are called
affinities. A high affinity means that the two pixels tend to belong
to the same segment. The classifier computes the affinity of each
edge based on an image patch surrounding the edge. 

The graph partitioner first thresholds the affinity graph by removing
all edges with weights less than some threshold value $\theta$. The
connected components of this thresholded affinity graph are the segments
of the image.

For this class of segmentation algorithms, it's obvious that a single
misclassified edge of the affinity graph can dramatically alter the
resulting segmentation by splitting or merging two segments (see Fig.
\ref{fig:Segmentation-algorithm-two-graphs}). This is why it is important
to learn by optimizing a measure of segmentation performance rather
than affinity prediction.

We are well aware that connected components is an exceedingly simple
method of graph partitioning. More sophisticated algorithms, such
as spectral clustering \cite{Shi:2000} or graph cuts \cite{Boykov:2001},
might be more robust to misclassifications of one or a few edges of
the affinity graph. Why not use them instead? We have two replies
to this question.

First, because of the simplicity of our graph partitioning, we can
derive a simple and direct method of supervised learning that optimizes
a true measure of image segmentation performance. So far learning
based on more sophisticated graph partitioning methods has fallen
short of this goal \cite{Meila:2001,Bach:2006}.

Second, even if it were possible to properly learn the affinities
used by more sophisticated graph partitioning methods, we would still
prefer our simple connected components. The classifier in our segmentation
algorithm can also carry out sophisticated computations, if its representational
power is sufficiently great. Putting the sophistication in the classifier
has the advantage of making it learnable, rather than hand-designed.

The sophisticated partitioning methods clean up the affinity graph
by using prior assumptions about the properties of image segmentations.
But these prior assumptions \emph{could be incorrect}. The spirit
of the machine learning approach is to use a large amount of training
data and minimize the use of prior assumptions. If the sophisticated
partitioning methods are indeed the best way of achieving good segmentation
performance, we suspect that our classifier will learn them from the
training data. If they are not the best way, we hope that our classifier
will do even better.

\section{The Rand index quantifies segmentation performance}

Image segmentation can be viewed as a special case of the general
problem of clustering, as image segments are clusters of image pixels.
Long ago, Rand proposed an index of similarity between two clusterings
\cite{Rand:1971}. Recently it has been proposed that the Rand index
be applied to image segmentations \cite{Unnikrishnan:2007}. Define
a segmentation $S$ as an assignment of a segment label $s_{i}$ to
each pixel $i$. The indicator function $\delta(s_{i},s_{j})$ is
$1$ if pixels $i$ and $j$ belong to the same segment ($s_{i}=s_{j}$)
and $0$ otherwise. Given two segmentations $S$ and $\hat{S}$ of
an image with $N$ pixels, define the function\begin{equation}
1-RI(\hat{S},S)={N \choose 2}^{-1}\sum_{i<j}\left|\delta(s_{i},s_{j})-\delta(\hat{s}_{i},\hat{s}_{j})\right|\label{eq:clustering-cost-function}\end{equation}
which is the fraction of image pixel pairs on which the two segmentations
\emph{disagree}. We will refer to the function $1-RI(\hat{S},S)$
as the Rand index, although strictly speaking the Rand index is $RI(\hat{S},S)$,
the fraction of image pixel pairs on which the two segmentations \emph{agree}.
In other words, the Rand index is a measure of similarity, but we
will often apply that term to a measure of dissimilarity. 

In this paper, the Rand index is applied to compare the output $\hat{S}$
of a segmentation algorithm with a ground truth segmentation $S$,
and will serve as an objective function for learning. Figure \ref{fig:Segmentation-algorithm-two-graphs}
illustrates why the Rand index is a sensible measure of segmentation
performance. The segmentation of affinity graph \#1 incurs a huge
Rand index penalty relative to the ground truth. A single wrongly
classified edge of the affinity graph leads to an incorrect merger
of two segments, causing many pairs of image pixels to be wrongly
assigned to the same segment. On the other hand, the segmentation
corresponding to affinity graph \#2 has a perfect Rand index, even
though there are misclassifications in the affinity graph. In short,
the Rand index makes sense because it strongly penalizes errors in
the affinity graph that lead to split and merger errors.

\begin{figure}
\begin{centering}
\includegraphics{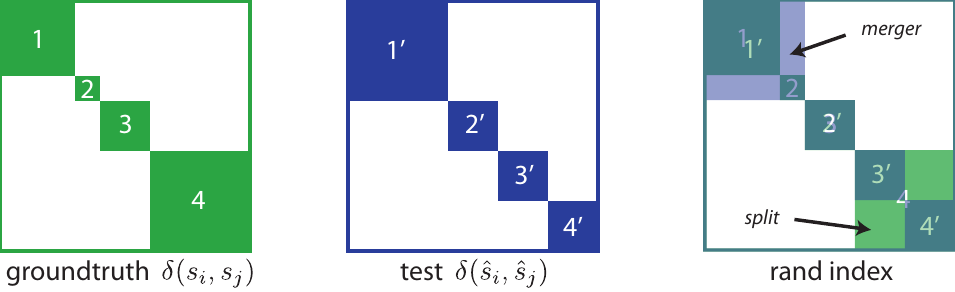}
\par\end{centering}

\caption{\textbf{The Rand index quantifies segmentation performance by comparing
the difference in pixel pair connectivity between the groundtruth
and test segmentations.} Pixel pair connectivities can be visualized
as symmetric binary block-diagonal matrices $\delta(s_{i},s_{j})$.
Each diagonal block corresponds to connected pixel pairs belonging
to one of the image segments. The Rand index incurs penalties when
pixels pairs that must not be connected are connected or vice versa.
This corresponds to locations where the two matrices disagree. An
erroneous merger of two groundtruth segments incurs a penalty proportional
to the product of the sizes of the two segments. Split errors are
similarly penalized.}

\end{figure}

\section{Connectivity and maximin affinity}

Recall that our segmentation algorithm works by finding connected
components of the thresholded affinity graph. Let $\hat{S}$ be the
segmentation produced in this way. To apply the Rand index to train
our classifier, we need a simple way of relating the indicator function
$\delta(\hat{s}_{i},\hat{s}_{j})$ in the Rand index to classifier
output. In other words, we would like a way of characterizing whether
two pixels are connected in the thresholded affinity graph.

To do this, we introduce the concept of maximin affinity, which is
defined for any pair of pixels in an affinity graph (the definition
is generally applicable to any weighted graph). Let $A_{kl}$be the
affinity of pixels $k$ and $l$. Let $\mathcal{P}{}_{ij}$ be the
set of all paths in the graph that connect pixels $i$ and $j$. For
every path $P$ in $\mathcal{P}_{ij}$, there is an edge (or edges)
with minimal affinity. This is written as $\min_{\left\langle k,l\right\rangle \in P}A_{kl}$,
where $\left\langle k,l\right\rangle \in P$ means that the edge between
pixels $k$ and $l$ are in the path $P$. 

A maximin path $P_{ij}^{\ast}$ is a path between pixels $i$ and
$j$ that maximizes the minimal affinity,\begin{eqnarray}
P_{ij}^{\ast} & = & \arg\max_{P\in\mathcal{P}_{ij}}\min_{\left\langle k,l\right\rangle \in P}A_{kl}\label{eq:maximin-path}\end{eqnarray}
The maximin affinity of pixels $i$ and $j$ is the affinity of the
maximin edge, or the minimal affinity of the maximin path, \begin{eqnarray}
A_{ij}^{\ast} & = & \max_{P\in\mathcal{P}{}_{ij}}\min_{\left\langle k,l\right\rangle \in P}A_{kl}\label{eq:maximin-distance}\end{eqnarray}
We are now ready for a trivial but important theorem.
\begin{thm}
A pair of pixels is connected in the thresholded affinity graph if
and only if their maximin affinity exceeds the threshold value. \end{thm}
\begin{proof}
By definition, a pixel pair is connected in the thresholded affinity
graph if and only if there exists a path between them. Such a path
is equivalent to a path in the unthresholded affinity graph for which
the minimal affinity is above the threshold value. This path in turn
exists if and only if the maximin affinity is above the threshold
value. 
\end{proof}
As a consequence of this theorem, pixel pairs can be classified as
connected or disconnected by thresholding maximin affinities. Let
$\hat{S}$ be the segmentation produced by thresholding the affinity
graph $A_{ij}$ and then finding connected components. Then the connectivity
indicator function is \begin{equation}
\delta(\hat{s}_{i},\hat{s}_{j})=H(A_{ij}^{*}-\theta)\label{eq:connectivity-condition}\end{equation}
where $H$ is the Heaviside step function.

Maximin affinities can be computed efficiently using minimum spanning
tree algorithms \cite{Fischer:2004}. A maximum spanning tree is equivalent
to a minimum spanning tree, up to a sign change of the weights. Any
path in a maximum spanning tree is a maximin path. For our nearest
neighbor affinity graphs, the maximin affinity of a pixel pair can
be computed in $O(|E|\cdot\alpha(|V|))$ where $|E|$ is the number
of graph edges and $|V|$ is the number of pixels and $\alpha(\cdot)$
is the inverse Ackerman function which grows sub-logarithmically.
The full matrix $A_{ij}^{*}$ can be computed in time $O(|V|^{2})$
since the computation can be shared. Note that maximin affinities
are required for training, but not testing. For segmenting the image
at test time, only a connected components computation need be performed,
which takes time linear in the number of edges $|E|$.

\section{Optimizing the Rand index by learning maximin affinities\label{sec:Classifier}}

Since the affinities and maximin affinities are both functions of
the image $I$ and the classifier parameters $W$, we will write them
as $A_{ij}(I;W)$ and $A_{ij}^{\ast}(I;W)$, respectively. By Eq.
(\ref{eq:connectivity-condition}) of the previous section, the Rand
index of Eq. ($\ref{eq:clustering-cost-function}$) takes the form\[
1-RI(S,I;W)={N \choose 2}^{-1}\sum_{i<j}\left|\delta(s_{i},s_{j})-H(A_{ij}^{\ast}(I;W)-\theta)\right|\]
Since this is a discontinuous function of the maximin affinities,
we make the usual relaxation by replacing $|\delta(s_{i},s_{j})-H(A_{ij}^{\ast}(I;W)-\theta)|$
with a continuous loss function $l(\delta(s_{i},s_{j}),A_{ij}^{\ast}(I;W))$.
Any standard loss such as the such as the square loss, $\frac{1}{2}(x-\hat{x})^{2}$,
or the hinge loss can be used for $l(x,\hat{x})$. Thus we obtain
a cost function suitable for gradient learning,

\begin{eqnarray}
E(S,I;W) & = & {N \choose 2}^{-1}\sum_{i<j}l(\delta(s_{i},s_{j}),A_{ij}^{\ast}(I;W))\nonumber \\
 & = & {N \choose 2}^{-1}\sum_{i<j}l(\delta(s_{i},s_{j}),\max_{P\in\mathcal{P}_{ij}}\min_{\left\langle k,l\right\rangle \in P}A_{kl}(I;W))\label{eq:clustering-cost-maximin}\end{eqnarray}

The max and min operations are continuous and differentiable (though
not continuously differentiable). If the loss function $l$ is smooth,
and the affinity $A_{kl}(I;W)$ is a smooth function, then the gradient
of the cost function is well-defined, and gradient descent can be
used as an optimization method.

Define $(k,l)=mm(i,j)$ to be the maximin edge for the pixel pair
$(i,j)$. If there is a tie, choose between the maximin edges at random.
Then the cost function takes the form \begin{eqnarray*}
E(S,I;W) & = & {N \choose 2}^{-1}\sum_{i<j}l(\delta(s_{i},s_{j}),A_{mm(i,j)}(I;W))\end{eqnarray*}
It's instructive to compare this with the cost function for standard
affinity learning\[
E_{standard}(S,I;W)=\frac{2}{cN}\sum_{\left\langle i,j\right\rangle }l(\delta(s_{i},s_{j}),A_{ij}(I;W))\]
where the sum is over all nearest neighbor pixel pairs $\left\langle i,j\right\rangle $
and $c$ is the number of nearest neighbors \cite{Fowlkes:2003}.
In contrast, the sum in the MALIS cost function is over all pairs
of pixels, whether or not they are adjacent in the affinity graph.
Note that a single edge can be the maximin edge for multiple pairs
of pixels, so its affinity can appear multiple times in the MALIS
cost function. Roughly speaking, the MALIS cost function is similar
to the standard cost function, except that each edge in the affinity
graph is weighted by the number of pixel pairs that it causes to be
incorrectly classified.

\section{Online stochastic gradient descent}

Computing the cost function or its gradient requires finding the maximin
edges for all pixel pairs. Such a batch computation could be used
for gradient learning. However, online stochastic gradient learning
is often more efficient than batch learning \cite{LeCun:1998}. Online
learning makes a gradient update of the parameters after each pair
of pixels, and is implemented as described in the box. 

\framebox{\begin{minipage}[t]{0.5\columnwidth}%

\subsection*{Maximin affinity learning}

1. Pick a random pair of (not necessarily nearest neighbor) pixels
$i$ and $j$ from a randomly drawn training image $I$.\\

2. Find a maximin edge $mm(i,j)$

3. Make the gradient update:\\
$W\leftarrow W+\eta\frac{d}{dW}l(\delta(s_{i},s_{j}),A_{mm(i,j)}(I;W))$%
\end{minipage}}%
\framebox{\begin{minipage}[t]{0.5\columnwidth}%

\subsection*{Standard affinity learning}

\noindent 1. Pick a random pair of nearest neighbor pixels $i$ and
$j$ from a randomly drawn training image $I$\\
\\

2. Make the gradient update:\\
$W\leftarrow W+\eta\frac{d}{dW}l(\delta(s_{i},s_{j}),A_{ij}(I;W))$%
\end{minipage}}

For comparison, we also show the standard affinity learning \cite{Fowlkes:2003}.
For each iteration, both learning methods pick a random pair of pixels
from a random image. Both compute the gradient of the weight of a
single edge in the affinity graph. However, the standard method picks
a nearest neighbor pixel pair and trains the affinity of the edge
between them. The maximin method picks a pixel pair of arbitrary separation
and trains the minimal affinity on a maximin path between them. 

Effectively, our connected components performs spatial integration
over the nearest neighbor affinity graph to make connectivity decisions
about pixel pairs at large distances. MALIS trains these global decisions,
while standard affinity learning trains only local decisions. MALIS
is superior because it truly learns segmentation, but this superiority
comes at a price. The maximin computation requires that on each iteration
the affinity graph be computed for the whole image. Therefore it is
slower than the standard learning method, which requires only a local
affinity prediction for the edge being trained. Thus there is a computational
price to be paid for the optimization of a true segmentation error.

\section{Application to electron microscopic images of neurons}

\subsection{Electron microscopic images of neural tissue\label{sub:EM}}

By 3d imaging of brain tissue at sufficiently high resolution, as
well as identifying synapses and tracing all axons and dendrites in
these images, it is possible in principle to reconstruct connectomes,
complete {}``wiring diagrams'' for a brain or piece of brain \cite{Seung:2009,Briggman:2006,Smith:2007}.
Axons can be narrower than 100 nm in diameter, necessitating the use
of electron microscopy (EM) \cite{Seung:2009}. At such high spatial
resolution, just one cubic millimeter of brain tissue yields teravoxel
scale image sizes. Recent advances in automation are making it possible
to collect such images \cite{Seung:2009,Briggman:2006,Smith:2007},
but image analysis remains a challenge. Tracing axons and dendrites
is a very large-scale image segmentation problem requiring high accuracy.
The images used for this study were from the inner plexiform layer
of the rabbit retina, and were taken using Serial Block-Face Scanning
Electron Microscopy \cite{Denk:2004}. Two large image volumes of
$100^{3}$ voxels were hand segmented and reserved for training and
testing purposes.

\subsection{Training convolutional networks for affinity classification}

Any classifier that is a smooth function of its parameters can be
used for maximin affinity learning. We have used convolutional networks
(CN), but our method is not restricted to this choice. Convolutional
networks have previously been shown to be effective for similar EM
images of brain tissue \cite{Jain:2007}.

We trained two identical four-layer CNs, one with standard affinity
learning and the second with MALIS. The CNs contained 5 feature maps
in each layer with sigmoid nonlinearities. All filters in the CN were
$5\times5\times5$ in size. This led to an affinity classifier that
uses a $17\times17\times17$ cubic image patch to classify a affinity
edge. We used the square-square loss function $l(x,\hat{x})=x\cdot\max(0,1-\hat{x}-m)^{2}+(1-x)\cdot\max(0,\hat{x}-m)^{2}$,
with a margin $m=0.3$.

As noted earlier, maximin affinity learning can be significantly slower
than standard affinity learning, due to the need for computing the
entire affinity graph on each iteration, while standard affinity training
need only predict the weight of a single edge in the graph. For this
reason, we constructed a proxy training image dataset by picking all
possible $21\times21\times21$ sized overlapping sub-images from the
original training set. Since each $21\times21\times21$ sub-image
is smaller than the original image, the size of the affinity graph
needed to be predicted for the sub-image is significantly smaller,
leading to faster training. A consequence of this approximation is
that the maximum separation between image pixel pairs chosen for training
is less than about 20 pixels. A second means of speeding up the maximin
procedure is by pretraining the maximin CN for 500,000 iterations
using the fast standard affinity classification cost function. At
the end, both CNs were trained for a total of 1,000,000 iterations
by which point the training error plateaued.

\subsection{Maximin learning leads to dramatic improvement in segmentation performance}

\begin{figure}[H]
\begin{centering}
\includegraphics[width=1\columnwidth]{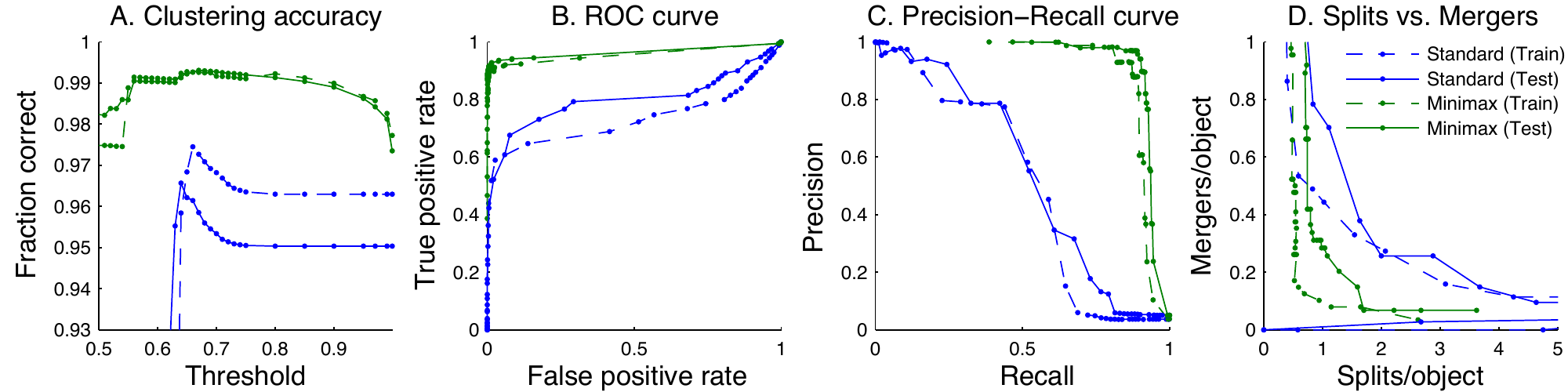}
\par\end{centering}

\centering{}\caption{\textbf{Quantification of segmentation performance on 3d electron
microscopic images of neural tissue.} A) Clustering accuracy measuring
the number of correctly classified pixel pairs. B) and C) ROC curve
and precision-recall quantification of pixel-pair connectivity classification
shows near perfect performance. D) Segmentation error as measured
by the number of splits and mergers. \label{fig:Quantification-of-segmentation}}

\end{figure}

We benchmarked the performance of the standard and maximin affinity
classifiers by measuring the the pixel-pair connectivity classification
performance using the Rand index. After training the standard and
MALIS affinity classifiers, we generated affinity graphs for the training
and test images. In principle, the training algorithm suggests a single
threshold for the graph partitioning. In practice, one can generate
a full spectrum of segmentations leading from over-segmentations to
under-segmentations by varying the threshold parameter. In Fig. \ref{fig:Quantification-of-segmentation},
we plot the Rand index for segmentations resulting from a range of
threshold values.

In images with large numbers of segments, most pixel pairs will be
disconnected from one another leading to a large imbalancing the number
of connected and disconnected pixel pairs. This is reflected in the
fact that the Rand index is over 95\% for both segmentation algorithms.
While this imbalance between positive and negative examples is not
a significant problem for training the affinity classifier, it can
make comparisons between classifiers difficult to interpret. Instead,
we can use the ROC and precision-recall methodologies, which provide
for accurate quantification of the accuracy of classifiers even in
the presence of large class imbalance. From these curves, we observe
that our maximin affinity classifier dramatically outperforms the
standard affinity classifier.

Our positive results have an intriguing interpretation. The poor performance
of the connected components when applied to a standard learned affinity
classifier could be interpreted to imply that 1) a local classifier
lacks the context important for good affinity prediction; 2) connected
components is a poor strategy for image segmentation since mistakes
in the affinity prediction of just a few edges can merge or split
segments. On the contrary, our experiments suggest that when trained
properly, thresholded affinity classification followed by connected
components can be an extremely competitive method of image segmentations.

\begin{figure}
\begin{centering}
\includegraphics{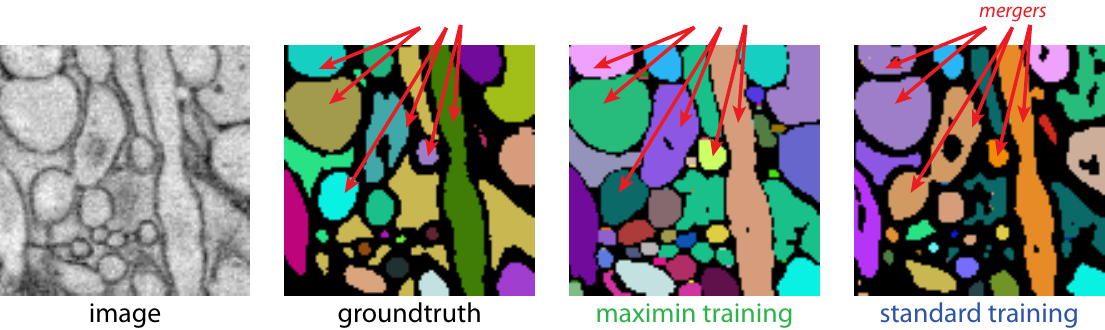}
\par\end{centering}

\caption{\textbf{A 2d cross-section through a 3d segmentation of the test image.}
The maximin segmentation correctly segments several objects which
are merged in the standard segmentation, and even correctly segments
objects which are missing in the groundtruth segmentation. Not all
segments merged in the standard segmentation are merged at locations
visible in this cross section. Pixels colored black in the machine
segmentations correspond to pixels completely disconnected from their
neighbors and represent boundary regions.}

\end{figure}

\section{Discussion\label{sec:Discussion}}

In this paper, we have trained an affinity classifier to produce affinity
graphs that result in excellent segmentations when partitioned by
the simple graph partitioning algorithm of thresholding followed by
connected components. The key to good performance is the training
of a segmentation-based cost function, and the use of a powerful trainable
classifier to predict affinity graphs. Once trained, our segmentation
algorithm is fast. In contrast to classic graph-based segmentation
algorithms where the partitioning phase dominates, our partitioning
algorithm is simple and can partition graphs in time linearly proportional
to the number of edges in the graph. We also do not require any prior
knowledge of the number of image segments or image segment sizes at
test time, in contrast to other graph partitioning algorithms \cite{Felzenszwalb:2004,Shi:2000}.

The formalism of maximin affinities used to derive our learning algorithm
has connections to single-linkage hierarchical clustering, minimum
spanning trees and ultrametric distances. Felzenszwalb and Huttenlocher
\cite{Felzenszwalb:2004} describe a graph partitioning algorithm
based on a minimum spanning tree computation which resembles our segmentation
algorithm, in part. The Ultrametric Contour Map algorithm \cite{Arbelaez:2006}
generates hierarchical segmentations nearly identical those generated
by varying the threshold of our graph partitioning algorithm. Neither
of these methods incorporates a means for learning from labeled data,
but our work shows how the performance of these algorithms can be
improved by use of our maximin affinity learning.

\subsection*{Acknowledgements}

SCT and HSS were supported in part by the Howard Hughes Medical Institute
and the Gatsby Charitable Foundation.

\bibliographystyle{ieee}
\bibliography{/Users/sturaga/papers/papers}

\end{document}